\documentclass[12pt, final]{l4dc2022}

\title[Joint Synthesis of Safety Certificates and Safe Control Policies]{Joint Synthesis of Safety Certificate and Safe Control Policy using Constrained Reinforcement Learning}
\usepackage{times}
\usepackage{xcolor}
\usepackage{todonotes}
\usepackage{amsfonts}
\usepackage{algorithmic}
\usepackage{algorithm}
\usepackage{multirow}
\usepackage{wrapfig}
\usepackage{booktabs}

\newcommand{\Ss}{\mathcal{S}}

\newcommand{\Ll}{\mathcal{L}}

\newcommand{\E}{\mathbb{E}}

\newcommand{\policy}{\pi}
\newcommand{\piparas}{\theta}
\newcommand{\qparas}{w}
\newcommand{\lamparas}{\xi}
\newcommand{\state}{s}
\newcommand{\st}{{\state_t}}
\newcommand{\stp}{{\state_{t+1}}}
\newcommand{\action}{a}

\newcommand{\at}{{\action_t}}

\newtheorem{assumption}{Assumption}

\author{%
 \Name{Haitong Ma} \Email{maht19@mails.tsinghua.edu.cn}\\
 \addr School of Vehicle and Mobility, Tsinghua University, Beijing 100084, China
 \AND
 \Name{Changliu Liu} \Email{cliu6@andrew.cmu.edu}\\
 \addr Robotics Institute, Carnegie Mellon University, Pittsburgh, PA 15213, USA
 \AND
 \Name{Shengbo Eben Li} \Email{lishbo@tsinghua.edu.cn}\AND
 \Name{Sifa Zheng} \Email{zsf@tsinghua.edu.cn}\\
 \addr School of Vehicle and Mobility, Tsinghua University, Beijing 100084, China
 \AND
 \Name{Jianyu Chen} \Email{jianyuchen@tsinghua.edu.cn}\\
 \addr Institute of Interdisciplinary Information Science, Tsinghua University, Beijing 100084, China\\
 \addr Shanghai Qizhi Institute, Shanghai 200000, China
}

\begin{document}

\maketitle

\begin{abstract}%
Safety is the major consideration in controlling complex dynamical systems using reinforcement learning (RL), where the \emph{safety certificates} can provide provable safety guarantees. A valid safety certificate is an energy function indicating that safe states are with low energy, and there exists a corresponding \emph{safe control policy} that allows the energy function to always dissipate.
The safety certificates and the safe control policies are closely related to each other and both challenging to synthesize. Therefore, existing learning-based studies treat either of them as prior knowledge to learn the other, limiting their applicability to general systems with unknown dynamics. This paper proposes a novel approach that simultaneously synthesizes the energy-function-based safety certificates and learns the safe control policies with constrained reinforcement learning (CRL). We do not rely on prior knowledge about \emph{either} a prior control law \emph{or} a perfect safety certificate. 
In particular, we formulate a loss function to optimize the safety certificate parameters by minimizing the occurrence of energy increases. By adding this optimization procedure as an outer loop to the Lagrangian-based CRL, we jointly update the policy and safety certificate parameters, and prove that they will converge to their respective local optima, the optimal safe policies and valid safety certificates.  Finally, we evaluate our algorithms on multiple safety-critical benchmark environments. The results show that the proposed algorithm learns solidly safe policies with no constraint violation. The validity or feasibility of synthesized safety certificates is also verified numerically.
\end{abstract}


\section{Introduction}
Safety is critical when applying state-of-the-art artificial intelligence studies to real-world applications, like autonomous driving \citep{sallab2017deep,chen2021interpretable}, robotic control \citep{richter2019open,ray2019benchmarking}.
Safe control is one of the most common tasks among these real-world applications, requiring that the hard safety constraints must be obeyed persistently. However, learning a safe control policy is hard for the naive trial-and-error mechanism of RL since it penalizes the dangerous actions \emph{after} experiencing them.

Meanwhile, in the control theory, there exist studies about \emph{energy-function-based} provable safety guarantee of dynamic systems called the \emph{safety certificate}, or \emph{safety index} \citep{wieland2007constructive,ames2014control,liu2014control}. These methods first synthesize energy functions such that the safe states have low energy, then design control laws satisfying the \emph{safe action constraints} to make the systems dissipate energy \citep{wei2019safe}. 
If there exists a feasible policy for \emph{all states} in a safe set to satisfy the \emph{safe action constraints} dissipating the energy, the system will never leave the safe set (i.e., forward invariance). 
Despite its soundness, the safety index synthesis (SIS) by hand is extremely hard for complex systems, which stimulates a rapidly growing interest in learning-based SIS \citep{chang2020neural,saveriano2019learning,srinivasan2020synthesis,ma2021model,qin2021learning}. Nevertheless, These studies usually require known dynamical models (white-box, black-box or learning-calibrated) to design control laws. Furthermore, obtaining the policy satisfying safe action constraints is also challenging. Adding \emph{safety shields or layers} to obtain supervised RL policies is a common approach \citep{wang2017safety,Agrawal2017a,cheng2019end,taylor2020learning}, but these studies usually assume to know the valid safety certificates.

In general safe control tasks with unknown dynamics, one usually has access to \emph{neither} the control laws \emph{nor} perfect safety certificates, which makes the previous two kinds of studies fall into a paradox---they rely on each other as the prior knowledge. Therefore, this paper proposes a novel algorithm without prior knowledge about model-based control laws or valid safety certificates. We define a loss function for SIS by minimizing the occurrence of energy increases. Then we formulate a CRL problem (rather than the commonly used shield methods) to unify the loss functions of SIS and CRL. By adding SIS as an outer loop to the Lagrangian-based solution to CRL, we jointly update the policies and safety certificates, and prove that they will converge to their respective local optima, the optimal safe policies and the valid safety certificates.

\textit{\textbf{Contributions.}} Our main contributions are: 1. We propose an algorithm of joint CRL and SIS that learns the safe policies and synthesizes the safety certificates simultaneously. This is the first algorithm requiring no prior knowledge of control laws or valid safety certificates. 2. We unify the loss function formulations of SIS and CRL. We therefore can form the multi-timescale adversarial RL training and prove its convergence. 3. We evaluate the proposed algorithm on multiple safety-critical benchmark environments Results demonstrate that we can simultaneously synthesize valid safety certificates and learn safe policies with zero constraint violation.

\subsection{Related works}
Representative energy-function-based safety certificates include barrier certificates \citep{prajna2007framework}, control barrier functions (CBF) \citep{wieland2007constructive}, safety set algorithm (SSA) \citep{liu2014control} and sliding mode methods \citep{gracia2013reactive}. 
Recent learning-based studies can be mainly divided into \emph{learning-based SIS} and \emph{learning safe control policies supervised by certificates}.   
\citet{chang2020neural,luo2021learning} use explicit models to rollout or project actions to satisfy safe action constraints. \citet{jin2020neural,qin2021learning} guide certificate learning with LQR controllers, \citet{anonymous2021modelfree} requires a black-box model to query online, and \citet{saveriano2019learning,srinivasan2020synthesis} use labeled data to fit certificates with supervised learning. 
The latter one, learning safe policy with supervisory usually assumed a valid safety certificate \citep{wang2017safety,cheng2019end,taylor2020learning}. It's a natural thought that one could learn the dynamic models to handle these issues \citep[like][]{cheng2019end,luo2021learning}, but learning models is much more complex than only learning policies, especially in RL tasks. 

\section{Problem Formulations}
We consider the safety specification that the system state $s$ should be constrained in a connected and closed set $\Ss_s$ which is called the \emph{safe set}. $\Ss_s$ should be a zero-sublevel set of a safety index function $\phi_0(\cdot)$ denoted by $\Ss_s = \{s|\phi_0(s)\leq0\}$. We study the Markov decision process (MDP) with deterministic dynamics (a reasonable assumption when dealing with safe control problems), defined by the tuple $(\Ss, \mathcal{A}, \mathcal{F}, r, c, \gamma,\phi)$, where $\Ss,\mathcal{A}$ is the state and action space, $\mathcal{F}: \Ss\times \mathcal{A}\to\Ss$ is the unknown system dynamics, $r,c:\mathcal{S} \times \mathcal{A} \times \mathcal{S} \rightarrow \mathbb{R}$ is the reward and cost function, $\gamma$ is the discounted factor, and $\phi:\Ss\to\mathbb{R}$ is the energy-function-based safety certificate, or called the \emph{safety index}. 
A safe control law with respect to safety index $\phi$ should keep the system energy low, ($\phi\leq0$) and dissipate the energy when the system is at high energy ($\phi> 0$). We use $s'$ to represent the next state for simplicity.
	 Then we can get the \emph{safe action constraint}:
	 \begin{definition}[Safe action constraint] For a given safety index $\phi$, the safe action constraint is
	 	\begin{equation}
	 		\phi(s')<\max \{\phi(s)-\eta_D, 0\}
	 		\label{eq:cstr0}
	 	\end{equation}
 	where $\eta_D$ is a slack variable controlling the descent rate of safety index. 
	 \end{definition}
 	\begin{definition}[Valid safety certificate]
 		If there always exists an action $a\in\mathcal{A}$ satisfying \eqref{eq:cstr0} at $s$, or the safe action set $\mathcal{U}_s(s)=\{a|\phi(s')<\max \{\phi(s)-\eta_D, 0\}\}$ is always nonempty, we say the safety index $\phi$ is a \textbf{valid}, or \textbf{feasible} safety certificate.
 	\end{definition}
    \begin{figure}
        \centering
        \includegraphics[width=0.5\linewidth]{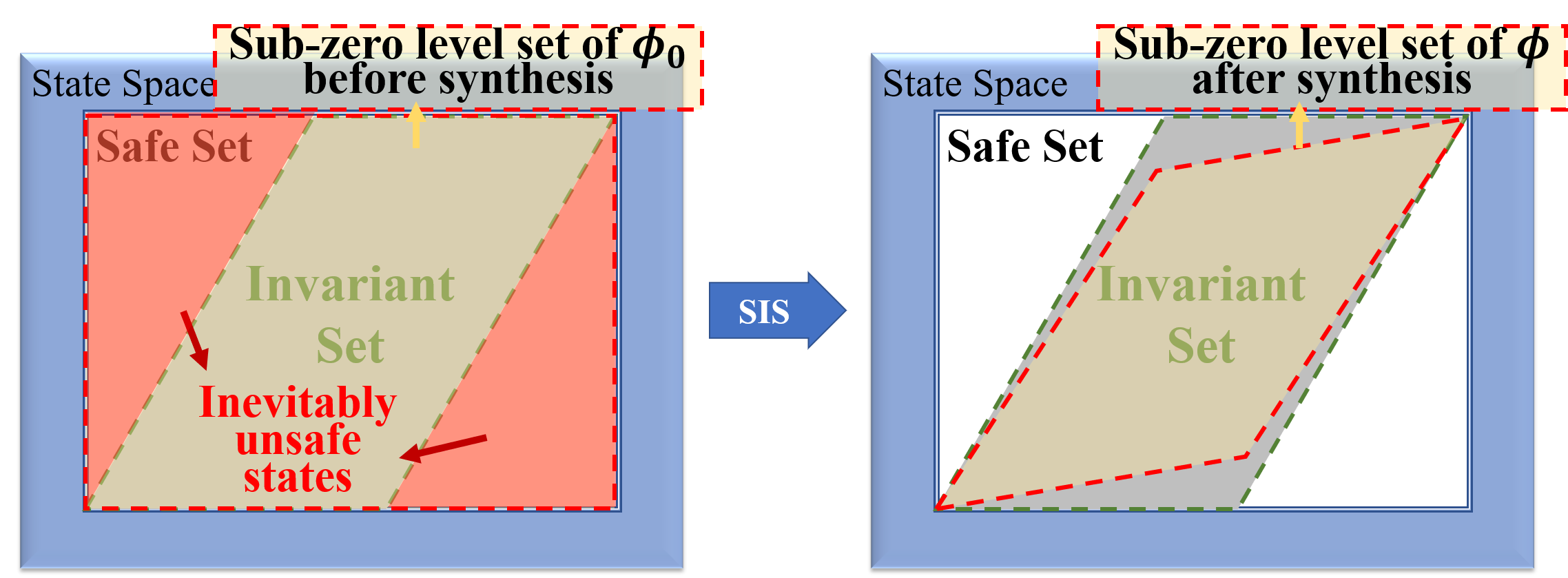}
        \caption{Safety index synthesis (SIS). Inevitably unsafe states should be excluded during the SIS.}
        \label{fig:my_label}
         \label{fig:intro}
    \end{figure}

	 A straightforward approach is to use the $\phi_0$ as the safety certificate. However, these safe action constraints are possibly not satisfied with all the states in $\Ss_s$ as shown in \figureref{fig:intro}. This problem is common in real-world tasks with actuator saturation and high relative-degree from safety specifications to control inputs (i.e., $\|\frac{\partial\phi_0}{\partial u}\|=0$). For example, if $\phi_0$ measures the distance between two autonomous vehicles, the collision may be inevitable because the relative speed is too high and brake force is limited. In this case, $\Ss_s$ includes inevitably unsafe states. We need to assign high energy values to these inevitably-unsafe states, for example, by linearly combining the $\phi_0$ and its high-order derivatives \citep{liu2014control}. The valid safety certificate will guarantee safety by ensuring the \emph{forward invariance} of a subset of $\Ss_s$. 
	 \begin{lemma}[Forward invariance \citep{liu2014control}]
	 	Define the zero-sublevel set of a valid safety index $\phi$ as $\Ss_s^\phi=\{s|\phi(s)\leq0\}$. If $\phi$ is a valid safety certificate, then there exist policies to guarantee the forward invariance of $\Ss_s^\phi\cap\Ss_s$.
	 \end{lemma}
	
	We therefore can formulate the CRL problem by adding the safe action constraints to RL optimization objective:
	\begin{equation}
		\max_{\pi}\  \E_{\tau\sim\pi}\Big\{\sum\nolimits_{t=0}^{\infty}\gamma^t r_t\Big\} = \E_{\state}\big\{v^{\pi}(\state)\big\}
		\quad	\text{s.t.} \    \phi(s')-\max \{\phi(s)-\eta_D, 0\}<0, \forall s \in \mathcal{D}
		\label{eq:statewiseop}
	\end{equation}
	where $v^\pi(s)$ is the state-value function of $s$, $\mathcal{D}=\{s|f(s)>0\}$ is the set of all possible states ($f$ is the state distribution density). 
	\begin{remark}
\eqref{eq:statewiseop} has \textbf{state-dependent} constraints; it can not be solved by previous model-free CRL since their constraint objectives are defined on the \textbf{expectation} over $\mathcal{D}$ \citep{uchibe2007constrained,achiam2017constrained,chow2017risk,tessler2018reward,ray2019benchmarking,stooke2020responsive}. 
\end{remark}

	We solve \eqref{eq:statewiseop} based on our previous framework to solve state-dependent safety constraints in \citep{ma2021feasible} using the Lagrange multiplier networks $\lambda(s)$ in a Lagrangian-based approach. The Lagrange function is \citep{ma2021feasible}
	\begin{equation}
		\mathcal{L}'(\pi,\lambda) = \E_{\state}\big\{-v^{\pi}(\state) + \lambda(\state)\big(\phi(s')-\max \{\phi(s)-\eta_D, 0\}\big)\big\}
		\label{eq:SL2}
	\end{equation}

	We can solve \eqref{eq:statewiseop} by locating the saddle point of $\Ll'(\pi,\lambda)$, $\max _{\lambda} \min _{\pi}\Ll'(\pi,\lambda)$.

\section{Joint Synthesis of Safety Certificate and Safe Control Policy}
The key idea of this section is to unify the loss functions of CRL and SIS; we provide theoretical analyses of their equivalence.
\label{sec:joint}
\subsection{Loss Function for Safety Index Synthesis}

We construct the loss for optimizing a parameterized safety index by a measurement of the \emph{violation of constraint} \eqref{eq:cstr0} 
\begin{equation}
	J(\phi)=\left.\mathbb{E}_{s}\Big\{\left[\phi(s')-\max \{\phi(s)-\eta_D, 0\}\right]^{+}\Big\}\right|_{\pi=\pi^*}
	\label{eq:philoss1}
\end{equation}
where $[\cdot]^+$ means projecting the values to the positive half-space $[0, +\infty)$, $\pi^*$ is the optimal safe policy (also a feasible policy when $\phi$ is a valid safety index) of \eqref{eq:statewiseop}, and $\cdot|_{\pi=\pi^*}$ represents the agent takes $\pi^*(s)$ to reach $s'$.
Ideally, if $\phi$ is a valid safety index, there always exists control to satisfy \eqref{eq:cstr0}, and $J(\phi)=0$. For those imperfect $\phi$, the inequality constraint in \eqref{eq:statewiseop} may not hold for all states in $\mathcal{D}$, so we can optimize the loss to get better $\phi$. 

The joint synthesis algorithm is tricky since we need to handle \emph{two different optimization problems}, \eqref{eq:SL2} and \eqref{eq:philoss1}. Recent similar studies integrate two optimizations by weighted sum \citep{qin2021learning} or alternative update \citep{luo2021learning}, but their methods are more like intuitive approaches and lack a solid theoretical basis.

\subsection{Unified Loss Function for Joint Synthesis}

\begin{lemma}[Statewise complementary slackness condition \citep{ma2021feasible}]
	\label{prop:scsc}
	For the problem \eqref{eq:statewiseop}, if the safe action set is not empty at state $\state$, the optimal multiplier and optimal policy $\lambda^*,\pi^*$ satisfies
	\begin{equation}
		\lambda^*(\state)\{\phi(\state')- \max \{\phi(s)-\eta_D, 0\}|_{\pi*}\}=0
	\end{equation}
	If the safe action set is empty at state $s$, then $\lambda^*(s)\to \infty$.
\label{lemma:1}
\end{lemma} 
\noindent The lemma comes from the Karush-Kuhn-Tucker (KKT) necessary conditions for the problem \eqref{eq:statewiseop}. 

 Consider the Lagrange function \eqref{eq:SL2} with the additional variable $\phi$ to optimize, 
\begin{equation}
		\Ll'(\pi,\lambda,\phi) = 
		\mathbb{E}_{s}\left\{-v^\pi(s)+\lambda(s)\big(\phi(\state')
		- \max \{\phi(s)-\eta_D, 0\}\big)\right\}
	\label{eq:lagwithphi}
\end{equation}
we have the following lemma for the relationship between the loss function of policy and certificate synthesis
\begin{lemma}
	If $\lambda$ is clipped into a compact set $[0,\lambda_{\max}]$, where $\lambda_{\max}>\max_{s\in\{s|\mathcal{U}_s(s)\neq\emptyset\}}\lambda^*(s)$. Then 
	\begin{equation}
		\Ll'(\pi^*,\lambda^*,\phi)= \lambda_{\max} J(\phi) + \Delta
	\end{equation}
	where $\Delta$ is a constant irrelevant with $\phi$.
	\label{lemma:propto}
\end{lemma}

\begin{proof}
    See \appendixref{sec:proofpropto}\footnote{The full paper with Appendix can be found on \url{https://arxiv.org/abs/2111.07695}.}. 
\end{proof}
\begin{theorem}[Unified loss for joint synthesis] The optimal safety certificate parameters with optimal policy-multiplier tuple in \eqref{eq:lagwithphi} is also the optimal safety certificate parameters under loss in \eqref{eq:philoss1}
	\begin{equation}
		\arg\min J(\phi) = \arg\min \Ll'(\pi^*,\lambda^*,\phi)
		\label{eq:equal}
	\end{equation}
	\label{theorem:loss}
\end{theorem}
\begin{proof}
    See the \emph{envelope theorem} on parameterized constraints in \citet{milgrom2002envelope,rockafellar2015convex}.
\end{proof}
Finally, we unify the loss function of updating three elements: policy $\pi$, multiplier $\lambda$, and safety index function $\phi$.
The optimization problem is formulated by a multi-timescale adversarial training:
\begin{equation}
	\min _{\phi} \max _{\lambda} \min _{\pi} \Ll'(\pi,\lambda,\phi)
	\label{eq:minmaxmin}
\end{equation}
\section{Practical Algorithm using Constrained Reinforcement Learning}
In this section, we explain the practical algorithm and convergence analysis.
\label{sec:algo}
\subsection{Algorithm Details}
The Lagrangian-based solution to CRL with statewise safety constraint is compatible with any existing unconstrained RL baselines, and we select the off-policy maximum entropy RL framework like soft actor-critic \citep[SAC,][]{haarnoja2018soft}.  According to \eqref{eq:SL2}, we need to add two neural networks to learn a state-action value function for safety index model $Q_\phi(s, a)$ (to approximate $\phi(\state') - \max \{\phi(s)-\eta_D, 0\}$) and the multipliers $\lambda(s)$. Because the soft policy evaluation, or the soft Q-learning part, is the same as \citet{haarnoja2018soft}, we only demonstrate the soft policy iteration part of the algorithm in Algorithm \ref{alg:facspi}. We name the proposed algorithm as FAC-SIS\footnote{FAC refers to the FAC algorithm in our prior work\citep{ma2021feasible}.}. We denote the parameters of the policy network, multiplier network, safety index as $\theta,\xi,\zeta$, respectively. We use $G_\#,\ \#\in\{\theta,\xi,\zeta\}$ to denote the gradients to update policy, multiplier and certificate. Detailed computations of the gradients can be found in Appendix \ref{sec:grad}. In addition, we assign multiple delayed updates \citep[similar to][]{fujimoto2018addressing}, $m_\pi<m_\lambda<m_\phi$, to stabilize the adversarial optimizations. 
\begin{remark}
A general parameterization rule of safety index is to linearly combine $\phi_0$ and its high-order derivatives \citep{liu2014control}, $\phi=\phi_0 + k_1\dot{\phi_0} + \dots + k_n\phi_0^{(n)}$; the parameters $\xi$ is $[k_1,k_2,\dots,k_n]$. How many high-order derivatives are needed depends on the system relative-degree. For example, the relative-degree of position constraints with force inputs is $2$. This information should be included in observations of MDP. Otherwise, the observation can not fully describe how dangerous the agent is with respect to the safety constraint.
\end{remark}
\begin{small}
    \begin{algorithm}[h]
	\caption{Soft Policy Improvement in FAC-SIS}
	\label{alg:facspi}
	\begin{algorithmic}[1]
		\REQUIRE Buffer $\mathcal{D}$ with sampled data, policy parameters $\piparas$, multiplier parameters $\xi$, safety index parameters $\zeta$.
	    \STATE {\textbf{if} gradient steps \texttt{mod} $m_\pi$ $=0$ \textbf{then} } $\piparas \leftarrow \piparas - \overline{\beta_\policy} G_{\piparas}$
		\STATE {\textbf{if} gradient steps \texttt{mod} $m_\lambda$ $=0$ \textbf{then} } $\lamparas \leftarrow \lamparas + \overline{\beta_\lambda} G_{\lamparas}$
		\STATE {\textbf{if} gradient steps \texttt{mod} $m_\phi$ $=0$ \textbf{then} }  $\zeta \leftarrow \zeta - \overline{\beta_\zeta} G_{\zeta}$
		\ENSURE $\qparas_1$, $\qparas_2$, $\piparas$,  $\xi$.
	\end{algorithmic}
\end{algorithm}
\end{small}

\subsection{Convergence Analysis}
The convergence proof of a three timescale adversarial training of \eqref{eq:minmaxmin} mainly follows the multi-timescale convergence according to Theorem 2 in Chapter 6 in \citet{borkar2009stochastic} about multiple timescale convergence of multi-variable optimization. Some studies also adopted this procedure to explain the convergence of RL algorithms from the perspective of stochastic optimization \citep{bhatnagar2009natural,bhatnagar2012online}, especially those with Lagrangian-based methods \citep{chow2017risk}. We incorporate the recent study on \emph{clipped stochastic gradient descent} to further improve the generalization of this convergence proof \citep{zhang2019gradient}. We first give some assumptions:
\begin{assumption}[learning rates]
	The learning rate schedules, $\{\beta_\theta(k), \beta_\xi(k), \beta_\zeta(k)\}$, satisfy
	    \begin{equation}
	    \notag
		\begin{gathered}
			{\small \sum\nolimits_{k} \beta_\theta(k)=\sum\nolimits_{k} \beta_\xi(k)=\sum\nolimits_{k} \beta_\zeta(k)=\infty.} \\
			{\small \sum\nolimits_{k} \beta_\theta(k)^{2}, \quad \sum\nolimits_{k} \beta_\xi(k)^{2}, \quad \sum\nolimits_{k} \beta_\zeta(k)^{2}<\infty, \quad
			\beta_\xi(k)=o\left(\beta_\theta(k)\right),\ \beta_\zeta(k)=o\left(\beta_\xi(k)\right).}
		\end{gathered}
	\end{equation}
\end{assumption}
\noindent This assumption also implies that the policy converges in the fastest timescale, then the multipliers, and finally the safety index parameters. 
\begin{proposition}[Clipped gradient descent]
	The actual learning rate used in Algorithm 1 is 
	\begin{equation}
		\overline{\beta_{\#}}:=\min \left\{\beta_{\#}, { \beta_{\#}}/{\|G_{\#}\|}\right\}
	\end{equation}
	where $\#\in\{w_1,w_2,\theta,\xi,\zeta\}$ is the parameters, and $G_{\#}$ is the corresponding gradients.
\end{proposition}
\begin{assumption}
	The state and action are sampled from compact sets, and all neural networks are $L_0-L_1$ smooth.
	\label{assump:compact}
\end{assumption}
As we are to finish a safe control problem that the agent should be confined in safe sets, and the actuator has physical limits, the bounded assumption is reasonable. We use multi-layer perceptron with continuous differentiable activation functions in practical implementations (details can be found in Appendix \ref{sec:para}).
\begin{theorem}
	Under all the aforementioned assumptions, the sequence of policy, multiplier, and safety index parameters tuple $(\theta_k, \xi_k, \zeta_k)$ converge almost surely to a locally optimal safety index parameters and its corresponding  locally optimal policy and multiplier $(\theta^*, \xi^*, \zeta^*)$ as $k$ goes to infinity.
	\label{theorem:major}
\end{theorem}
\begin{proof}
    See Appendix \ref{appendix:proof1}. 

\end{proof}
\section{Experiments}
In our experiments, we focus on the following questions:
\begin{enumerate}
	\item How does the proposed algorithm compare with other constraint RL algorithms? Can it achieve a safe policy with zero constraint violation?
	\item How does the learning-based safety index synthesis outperform the handcrafted safety index or the original safety index in the safety performance?
	\item Does the synthesized safety index allow safe control in all states the agent experienced?
\end{enumerate}
To demonstrate the effectiveness of the proposed online synthesis rules, we select the safe reinforcement learning benchmark environments Safety Gym \citep{ray2019benchmarking} with different tasks and obstacles. We name a specific environment by \texttt{\{Obstacle type\}-\{Obstacle size\}-\{Task\}}. We select six environments with different tasks and constraint objectives, where four of them are demonstrated in Figure \ref{fig:safetygym}, and others are provided in \appendixref{sec:addexpgym}. 
\begin{remark}
Some previous studies use controllers specially designed for goal-reaching tasks \citep{jin2020neural,qin2021learning}, while our algorithm can handle arbitrary complex tasks like the \texttt{Push} tasks. 
\end{remark}
\begin{figure}[b]

	\centering
	\subfigure[{\small Hazards: non-physical dangerous areas.}]{\includegraphics[width=0.22\linewidth]{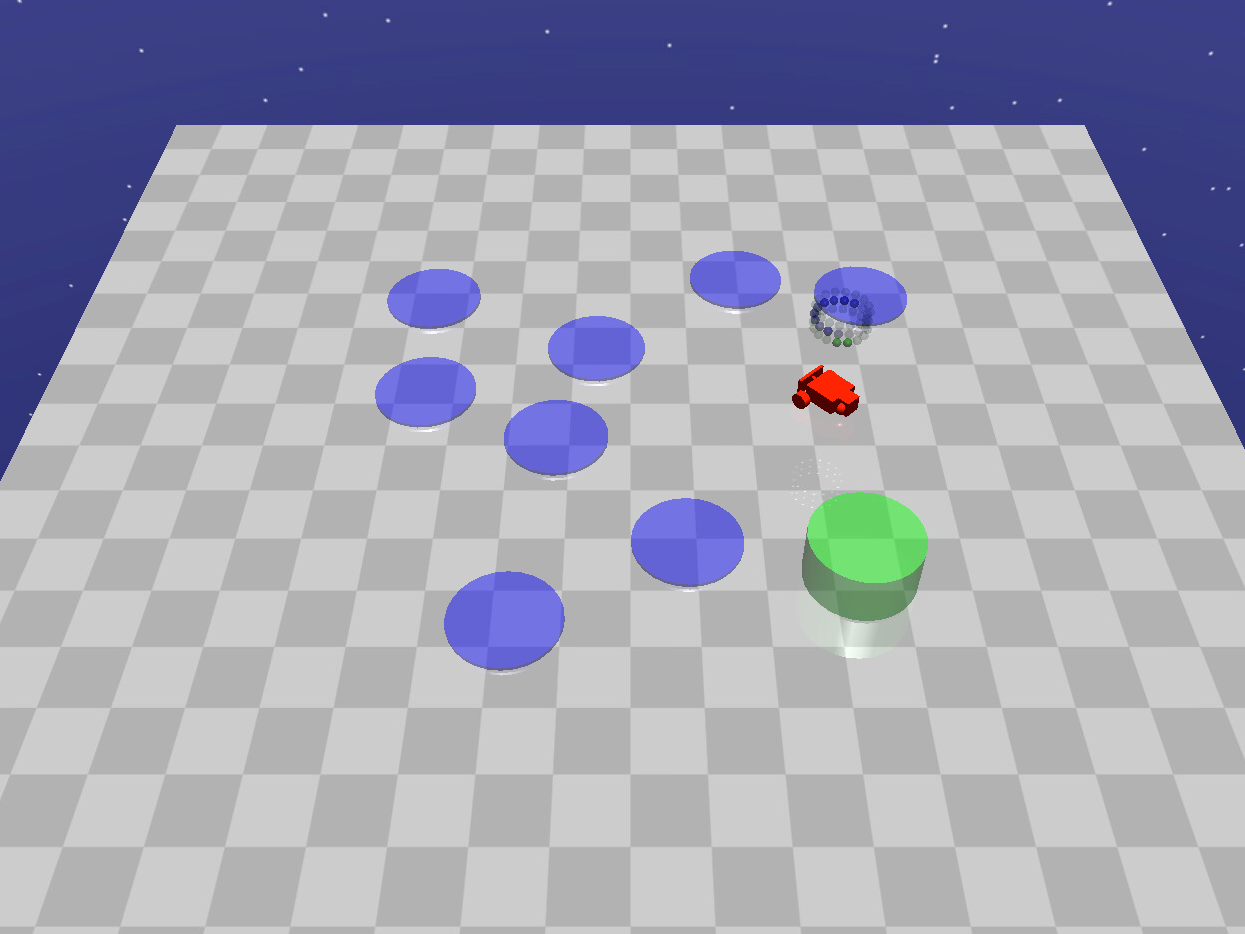}}
	\subfigure[{\small Pillars: Static physical obstacles}]{\includegraphics[width=0.22\linewidth]{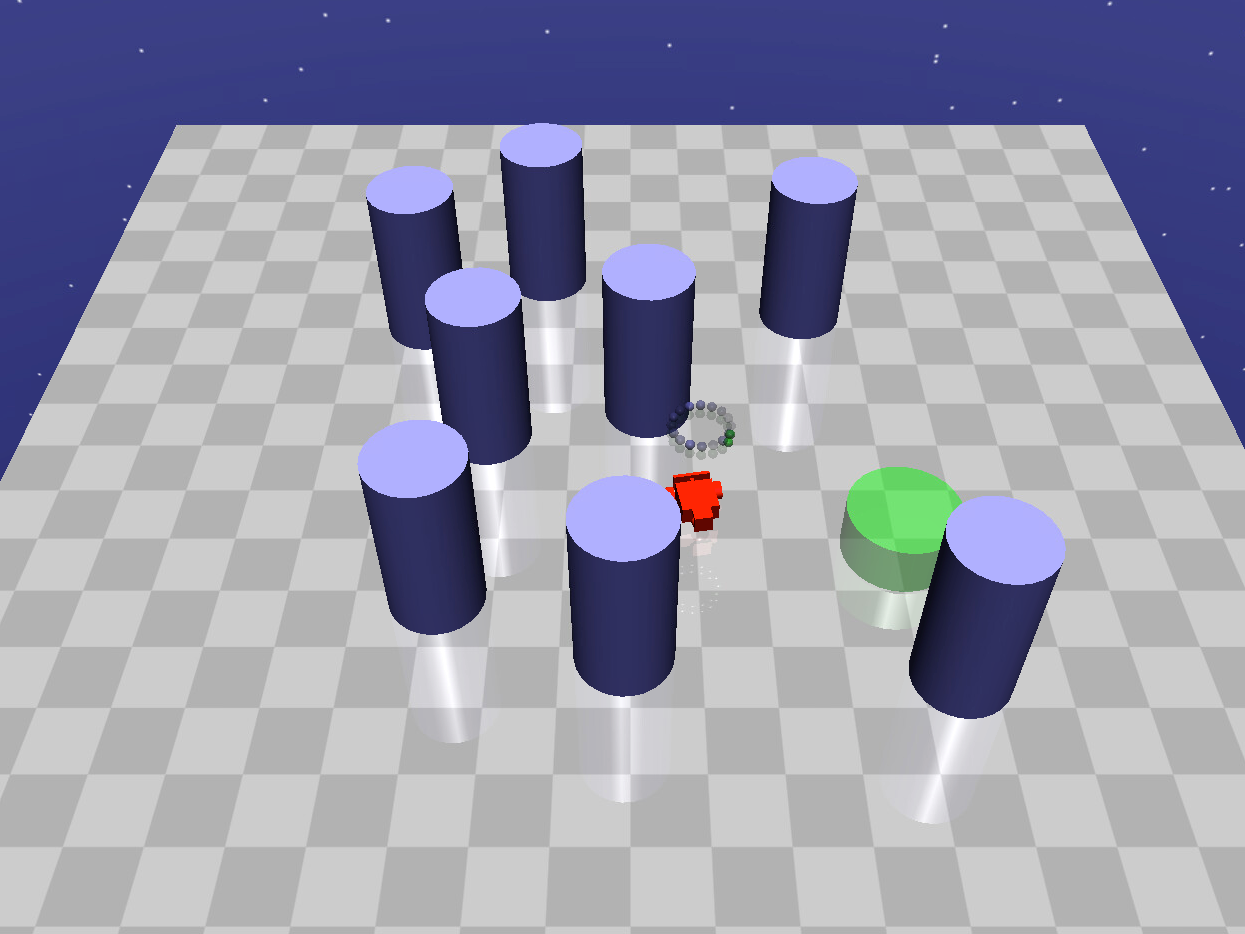}}
	\subfigure[{\small Goal task: navigating the robot to reach the green cylinder.}]{\includegraphics[width=0.22\linewidth]{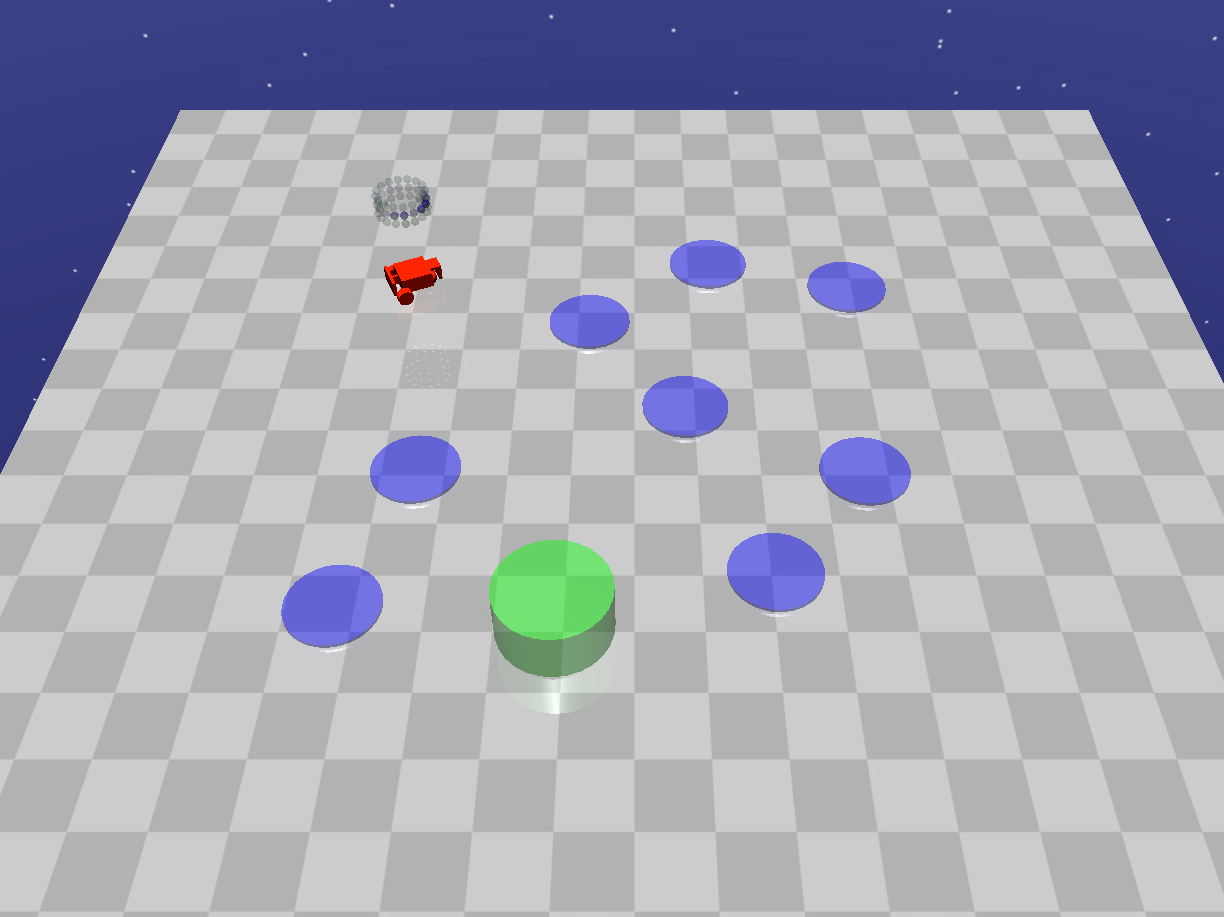}}
	\subfigure[{\small Push task: pushing the yellow box inside the green cylinder.}]{\includegraphics[width=0.22\linewidth]{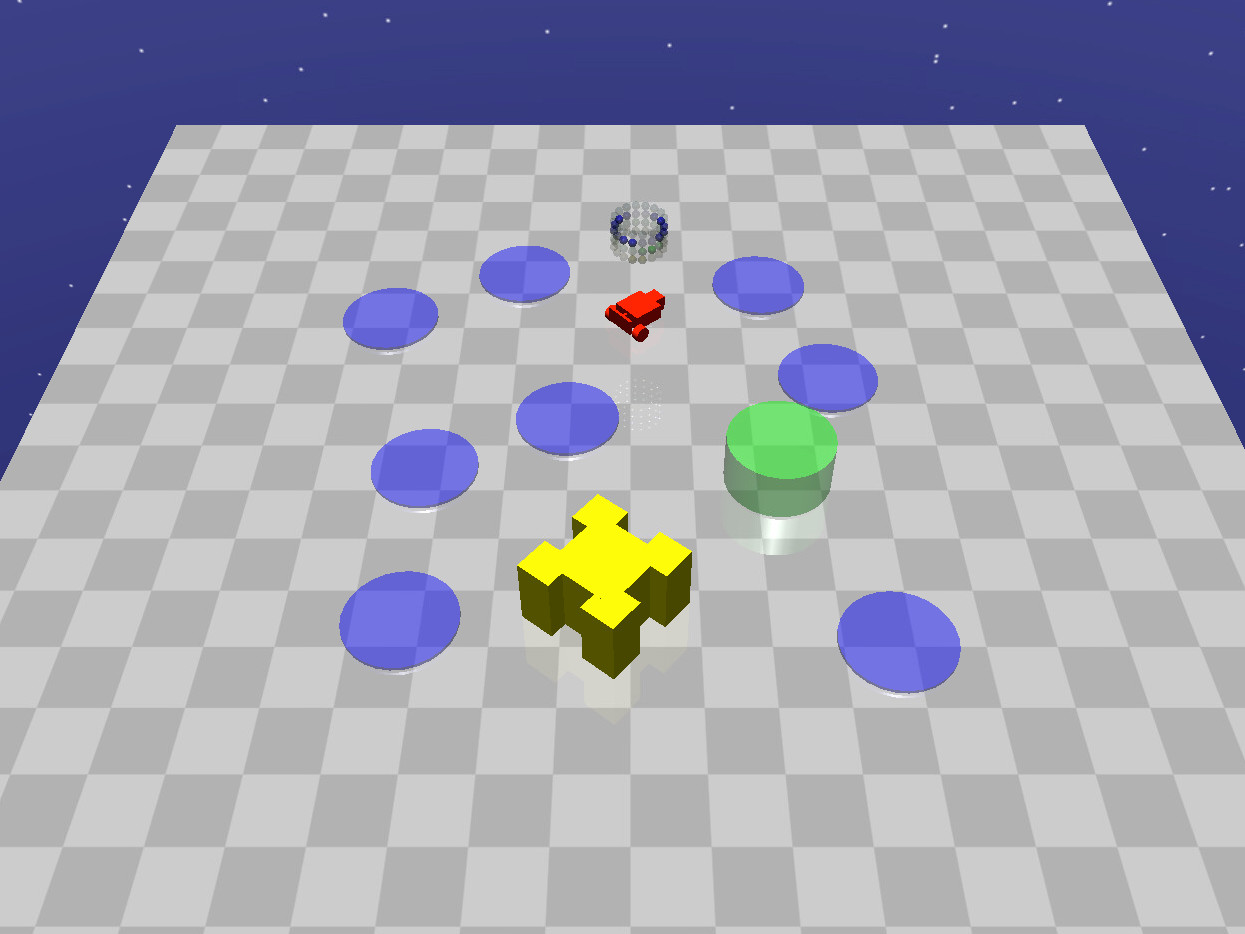}}
	\caption{Obstacles and tasks in Safety Gym.}
	\label{fig:safetygym}
\end{figure}
In this section, we use a fine-tuned form of safety index in \citet{anonymous2021modelfree}
\begin{equation}
	\phi(s)=\sigma+d_{\min }^{n}-d^{n}-k \dot{d}
	\label{eq:sis}
\end{equation}
where $d$ is the distance between the agent and obstacle, $d_{\min}$ is the minimum safe distance, and $\dot d$ is the derivative of distance with respect to time, $\xi=[\sigma, n, k]$ are the tunable parameters we desire to optimize in the online synthesis algorithm. The observations in Safety Gym include Lidar, speedometer, and magnetometer, which can be used to compute $d$ and $\dot d$ from observations. 
We compare the proposed algorithm against two types of baseline algorithms: 
\begin{itemize}
	\item CRL baselines, including TRPO-Lagrangian, PPO-Lagrangian and CPO \cite{achiam2017constrained,ray2019benchmarking}. The cost threshold is set at zero to learn solid safe policies.
	\item FAC with original safety index $\phi_0$ and handcrafted safety index $\phi_h$, where $\phi_0=d_{min}-d$ and $\phi_h=0.3 + d_{min}^2 - d^2 -k\dot d$, named as \emph{FAC with $\phi_0$} and \emph{FAC with $\phi_h$}.  The choice of $\phi_h$ is based on empirical knowledge. Details about baseline algorithms can be found in Appendix \ref{sec:implementationdetail}.
\end{itemize}
\subsection{Evaluating FAC-SIS and Comparison Analysis}
\begin{figure}[t]
    \centering
	\includegraphics[width=0.245\linewidth]{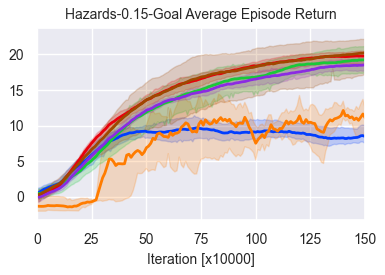}
	\includegraphics[width=0.245\linewidth]{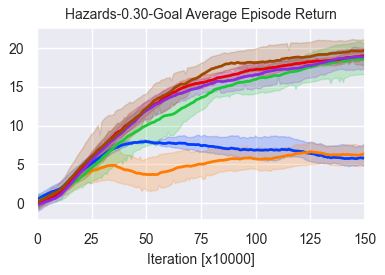}
	\includegraphics[width=0.245\linewidth]{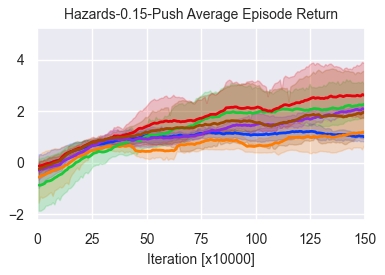}
	\includegraphics[width=0.245\linewidth]{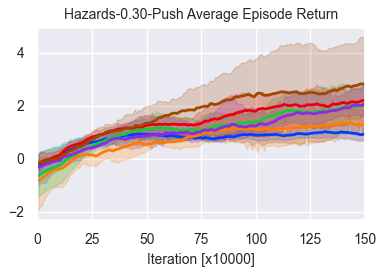}\\
	\includegraphics[width=0.245\linewidth]{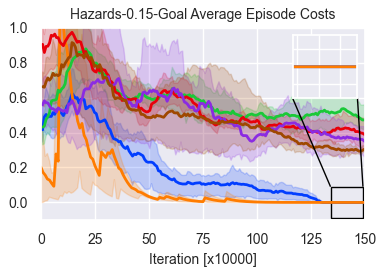}
	\includegraphics[width=0.245\linewidth]{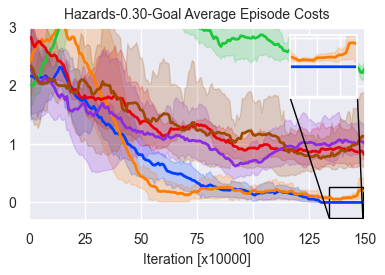}
	\includegraphics[width=0.245\linewidth]{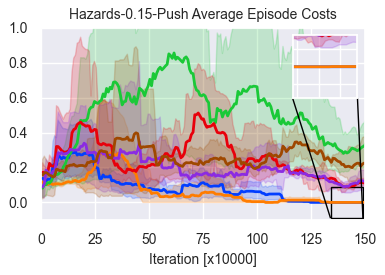}
	\includegraphics[width=0.245\linewidth]{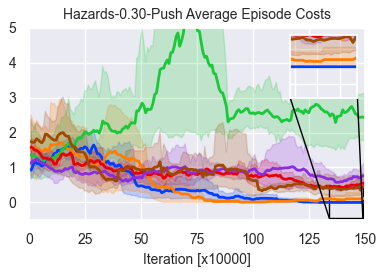}\\
	\includegraphics[width=0.245\linewidth]{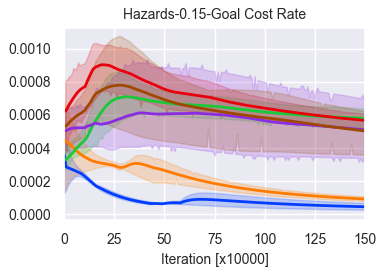}
	\includegraphics[width=0.245\linewidth]{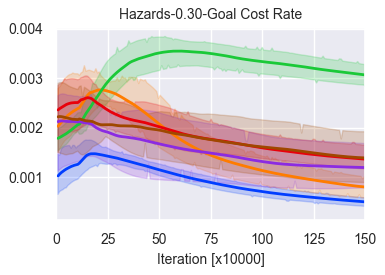}
	\includegraphics[width=0.245\linewidth]{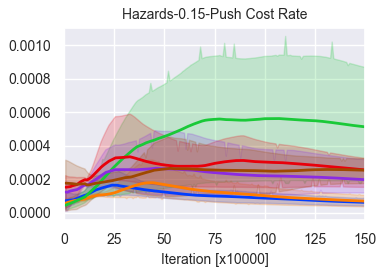}
	\includegraphics[width=0.245\linewidth]{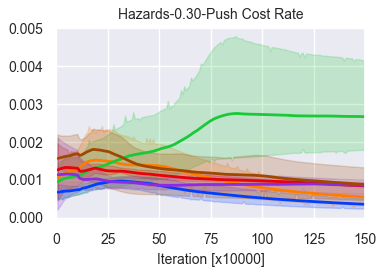}\\
	\includegraphics[width=0.6\linewidth]{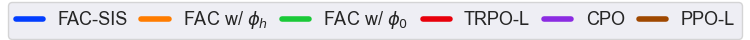}
	\caption{Average performance of FAC-SIS and baseline methods on 4 different Safety Gym environments over five seeds.}
	\label{fig:major}
\end{figure}
Results of the performance comparison are shown in Figure \ref{fig:major}. The results suggest that FAC with $\phi_0$ performs poorest in the safety performance, which indicates that there are indeed many inevitably unsafe states, and SIS is necessary for these tasks. Only FAC-SIS learns the policy with \emph{zero violation} and takes the lowest cost rate in all environments, answering the first question of zero constraint violation. FAC with $\phi_h$ fails to learn a solid safe policy in those environments with \texttt{0.30} constraint size (See the zoomed window in the second row), indicating that the handcrafted safety index can not cover all the environments.
As for the baseline CRL algorithms, they can not learn a zero-violation policy in any environment because of the posterior penalty in the trial-and-error mechanism stated above. As for the reward performance, FAC-SIS has comparable reward performance in the \texttt{Push} task. For the \texttt{Goal} task, FAC with $\phi_h$ and FAC-SIS sacrifice the reward performance to guarantee safety, explained in \citet{ray2019benchmarking}.

\subsection{Validity Verification of Synthesized Safety Index}
\begin{figure}[t]
	\centering
	\includegraphics[width=0.24\linewidth]{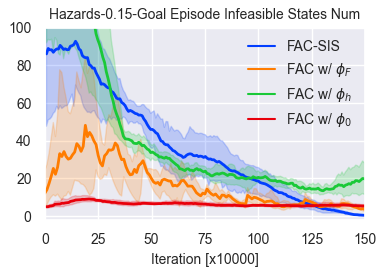}
	\includegraphics[width=0.24\linewidth]{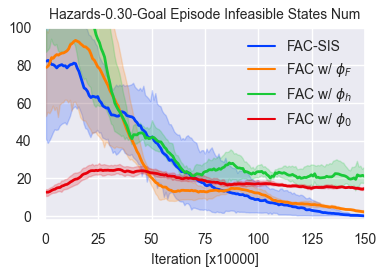}
	\includegraphics[width=0.24\linewidth]{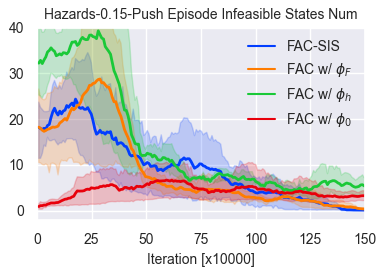}
	\includegraphics[width=0.24\linewidth]{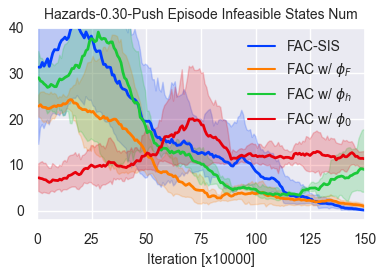}
	\caption{\small Average episodic number of violations of safe action constraint \eqref{eq:cstr0}. A valid safety index and its corresponding safe control policy should have zero violation performance.}
	\label{fig:phiincrease}
\end{figure}
We conduct new metrics and experiments to show the validity of our synthesized safety index. Recall that the validity means that there always exists a feasible control policy to satisfy the constraint in \eqref{eq:cstr0}. To effectively demonstrate the feasibility of SIS, we add a new baseline, \emph{FAC with} $\phi_F$, where $\phi_F$ is a valid safety index verified by \citet{anonymous2021modelfree}. Figure \ref{fig:phiincrease} demonstrates the episodic number of constraint violations of \eqref{eq:cstr0} in the Safety Gym environments. The results show that FAC-SIS and \emph{FAC with }$\phi_F$ can reach a nearly stable zero violation, which means that FAC can satisfy safe action constraint for a given valid safety index, and FAC-SIS also synthesizes a valid safety index. However, with $\phi_0$ and $\phi_h$ there are consistent violations even with the converged policy, caused by different reasons. For $\phi_h$, the reason is simply the inability to make the energy dissipate. For $\phi_0$, no high-order derivative in the safety index, so $\phi_0$ cannot handle the high relative-degree between the constraint function and control input. 

Furthermore, we want to give some scalable analyses about SIS. Firstly, we want to visualize that we can always find actions to dissipate the energy in the sampled distributions after SIS. 
We conduct experiments on a simple collision-avoidance environment shown in \figureref{fig:statedist}\footnote{See \appendixref{sec:statedist} for more details.}.
We select three different initialization rules of the agent and hazard. We use a sample-based exhaustive method to identify infeasible states and see if they overlap the sampled state distributions. We project the state distribution to the 2D space of $d$ and $\dot d$, and the results are listed in Figure \ref{fig:statedist}. As we expect all the sampled states to be feasible, these two distributions \emph{should not overlap}. The results show that the overlap of these two distributions is very small, which indicates that nearly all the sampling states are feasible.

Secondly, we want to visualize the shape of the safe set concerning the learned safety index. We slice the state space into three 2D planes with different agent heading angles, as shown in \figureref{fig:shape}. We use Hamilton-Jacobi reachability analysis to compute the avoidable sets numerically. The avoidable set considers the safe set under the \emph{most conservative} control inputs, which is the maximum controlled invariant set \citep{mitchell2007toolbox,choi2021robust}. We use $\phi_h$ as the initial safety index of 
SIS. \figureref{fig:shape} demonstrates that inevitable unsafe states exist in the zero-sublevel set of the empirical safety index. It also shows that we successfully exclude the inevitably unsafe sets through SIS. Notably, the zero-sublevel sets of the synthesized safety index are the subsets of the HJ-avoidable sets. The reasons why SIS can not learn the perfect shapes include limits of the representation capabilities of the safety index parameterization. Additionally, we still consider the \emph{optimal criteria}, resulting in less conservative policies and possibly smaller safe sets.
\begin{figure}
    \centering
    \includegraphics[width=0.2\linewidth]{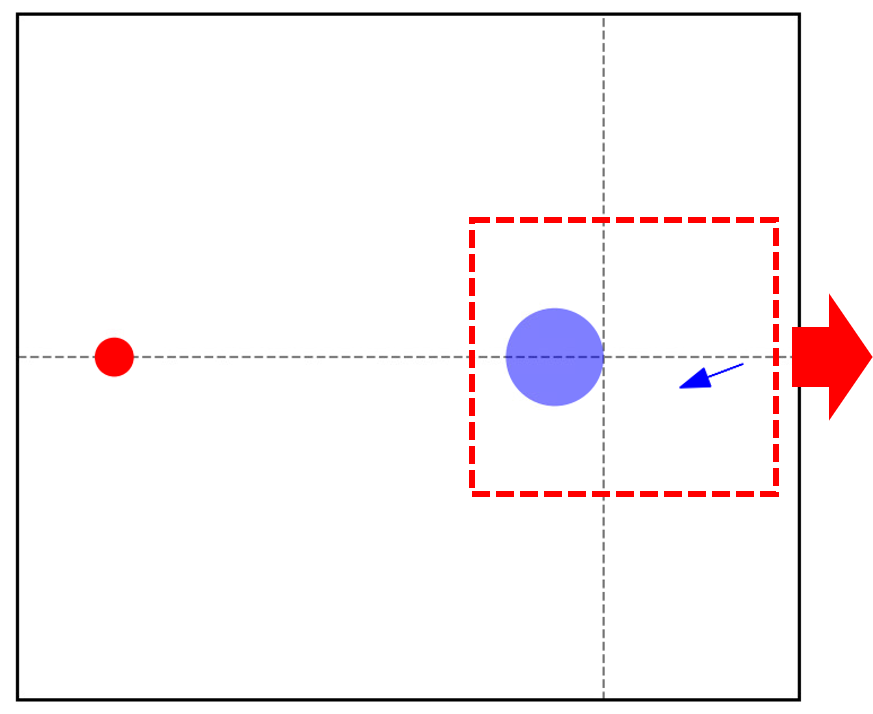}\quad
    \includegraphics[width=0.72\linewidth]{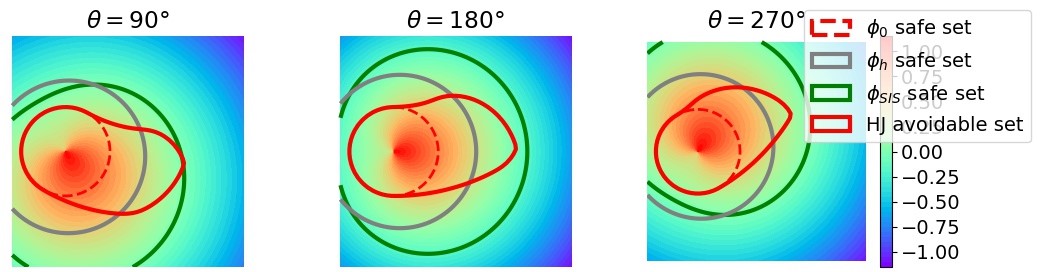}
    \caption{\small The color indicates the value of the synthesized safety index. The zero-sublevel set of learned safety index is a subset of the numerically max forward invariant set, the HJ avoidable set. We exclude all inevitably unsafe states through SIS like \figureref{fig:intro}. HJ-based solution is only capable for goal-reaching tasks with the reach-avoid problem formulation, while proposed algorithm can handle arbitrary task.}
    \label{fig:shape}
\end{figure}
\section{Conclusion}
This paper focuses on the joint synthesis of the safety certificate and the safe control policy for unknown dynamical systems and general tasks using a CRL approach. We are the first study to start with unknown dynamics and imperfect safety certificates, which significantly improves the applicability of the energy-function-based safe control. We add the optimization of safety index parameters as an outer loop of Lagrangian-based CRL with a unified loss. The convergence is analyzed theoretically. Experimental results demonstrate that the proposed FAC-SIS synthesizes a valid safe index while learning a safe control policy.
In future work, we will consider more complex safety index parameterization rules, for example, neural networks. Meanwhile, we will consider other factors in SIS, such as the reward performance of the safe control policies.
\acks{}
This work was done during Haitong Ma's internship at Carnegie Mellon University. This study is supported by National Key R\&D Program of China with 2020YFB1600602 and Tsinghua University-Didi Joint Research Center for Future Mobility. This study is also supported by Tsinghua-Toyota Joint Research Fund. The authors would like to thank Mr. Weiye Zhao and Mr. Tianhao Wei for their valuable suggestions on the experiments.
\bibliography{l4dc}
\newpage
\appendix

\section{Theoretical Results in Section \ref{sec:joint}}
\subsection{Proof of Lemma \ref{lemma:propto}}
\label{sec:proofpropto}
	\noindent Define set of states with safe actions as  $\Ss_f=\{s|\mathcal{U}_s(s)\neq\emptyset\}$. $\Delta=\mathbb{E}_s\{-v^{\pi^*}(s)\}$ since $-v^\pi(s)$ is irrelevant with $\phi$. For $s\notin\Ss_f$ (i.e., $\big(\phi(\state') - \max \{\phi(s)-\eta_D, 0\}\big)>0$), we know that $\lambda^*(s)\to \infty$ from Lemma \ref{prop:scsc}. Then $\lambda^*(s)$ is clipped to $\lambda_{\max}$. Therefore, the Lagrange function  $\eqref{eq:lagwithphi}$ can be reformulated to
\begin{equation}
	\begin{aligned}
		&\Ll'(\pi^*,\lambda^*,\phi) \\
		= &\left.\mathbb{E}_{s}\Big\{-v^\pi(s)+\lambda^*(s)\big(\phi(\state')- \max \{\phi(s)-\eta_D, 0\}\big)\Big\}\right|_{\pi=\pi^*}\\
		= &\left.\mathbb{E}_{s}\Big\{\lambda^*(s)\big(\phi(\state')- \max \{\phi(s)-\eta_D, 0\}\big)\Big\}\right|_{\pi=\pi^*} + \Delta\\
		= &\left.\mathbb{E}_{s\notin\Ss_f}\Big\{\lambda^*(s)\big(\phi(\state')- \max \{\phi(s)-\eta_D, 0\}\big)\Big\}\right|_{\pi=\pi^*} + \Delta\\
		= &\left.\lambda_{\max}\mathbb{E}_{s\notin\Ss_f}\Big\{\big(\phi(\state')- \max \{\phi(s)-\eta_D, 0\}\big)\Big\}\right|_{\pi=\pi^*} + \Delta\\
		= &\lambda_{\max}J(\phi) + \Delta
	\end{aligned}
	\label{eq:philoss2}
	\notag
\end{equation}
\section{Theoretical Results in Section \ref{sec:algo}}
\subsection{Gradients Computation}
\label{sec:grad}
The objective function of updating the policy and multipliers is the Lagrange function \eqref{eq:SL2}. Using the framework of maximum entropy RL, the objective function of policy update is:
 	\begin{equation}
 		J_{\pi}(\piparas)=\mathbb{E}_{\st \sim \mathcal{D}}\bigg\{\mathbb{E}_{\at \sim \pi_{\piparas}}\Big\{\alpha \log \big(\pi_{\piparas}\left(\at \mid \st\right)\big)-
 			Q_{\qparas}(\st, \at) + \lambda_\xi(\st) Q_\phi(\st,\at)\Big\}\bigg\}
 			\label{eq: losspi}
 			\notag
 	\end{equation}
 	The policy gradient with the reparameterized policy $\at=f_\piparas(\epsilon_t; \st)$ can be approximated by:
 	\begin{small}
 		\begin{equation*}
 			\begin{aligned}				G_\theta=\hat{\nabla}_{\piparas} J_{\pi}(\piparas)=&\nabla_{\piparas} \alpha \log \big(\pi_{\piparas}\left(\at \mid \st\right)\big)+\Big(\nabla_{\at} \alpha \log \left(\pi_{\piparas}\left(\at \mid \st\right)\right)\\
 				&-\nabla_{\at} \big(Q_\qparas\left(\st, \at\right) - \lambda_\xi(\st) Q_\phi\left(\st, \at\right)\big)\Big) \nabla_{\piparas} f_{\piparas}\left(\epsilon_{t} ; \st\right)
 			\end{aligned}
 		\end{equation*}
 	\end{small}
    where $\hat{\nabla}_{\piparas} J_{\pi}(\piparas)$ represents the stochastic gradient with respect to $\theta$, and $Q_\phi(s_t, a_t) = \big(\phi(\stp) - \max \{\phi(\st)-\eta_D, 0\}\big)$. Neglecting those irrelevant parts, the objective function of updating the multiplier network parameters $\xi$ is 
 	\begin{equation*}
 		J_{\lambda}(\xi) = \mathbb{E}_{\st \sim \mathcal{D}}\bigg\{\mathbb{E}_{\at \sim \pi_{\piparas}}\Big\{\lambda_\xi(\st)\big( Q_\phi(\st,\at)-d\big)\Big\}\bigg\}
 	\end{equation*}
 	The stochastic gradient is
 	\begin{equation}
 		G_\xi=\hat{\nabla}J_\lambda(\lamparas) = Q_\phi(\st,\at) \nabla_\lamparas\lambda_\lamparas(\st)
 		\label{eq:lamsubgrad}
 	\end{equation}
	The objective function of updating the safety index parameters $\zeta$ is already discussed, so the gradients for $\zeta$ is
 	\begin{equation}
 		G_{\zeta} = \lambda_{\xi}(\st)\nabla_\zeta \Delta\phi^{\pi_{\theta}}(\st)
 	\end{equation}
 	where $ \Delta\phi^{\pi_\theta}(\st)=\big(\phi(\stp) - \max \{\phi(\st)-\eta_D, 0\}\big)|_{\pi_\theta}$, we use a different notation from $Q_\phi(\st,\at)$ since we focus on different variables.
\subsection{Proof of Theorem \ref{theorem:major}}
\label{appendix:proof1}
Recall the overview of the total convergence proof:
\begin{enumerate}
	\item First we show that each update of the multi-time scale discrete stochastic approximation algorithm $(\theta_k, \xi_k,\zeta_k)$ converges almost surely, but at different speeds, to the \emph{stationary point} $(\theta^*, \xi^*, \zeta^*)$ of the corresponding continuous-time system.
	\item By Lyapunov analysis, we show that the continuous time system is locally asymptotically stable at $(\theta^*, \xi^*, \zeta^*)$.
	\item We prove that $(\theta^*, \xi^*)$ is the locally optimal solution, or a local saddle point for the CRL problems with local optimal safety index parameters $\zeta^*$.
\end{enumerate}
First, we introduce the important lemma used for the convergence proof:
\begin{lemma}[Soft policy evaluation, \citet{haarnoja2018softb}]
    The Q-function update will converge to the soft Q-function as the iteration number goes to infinity. 
\end{lemma}
\begin{remark}
 In stochastic programming, the error in policy improvement caused by Q-function is a fixed bias rather than random variables, resulting in that it will not affect the convergence as long as the error is bounded. Therefore, we assume that Q-function is fully updated here for simplicity. Otherwise, the proof will be wordy.
\end{remark}
\begin{lemma}[Convergence of clipped SGD, \cite{zhang2019gradient}]
	For a stochastic gradient descent problem of a continuous differentiable and $L_0-L_1$ smooth (which means $\left\|\nabla^{2} f(x)\right\| \leq L_{0}+L_{1}\|\nabla f(x)\|$) loss function $f(x)$ and its stochastic gradient $\hat\nabla f(x)$, If these conditions are satisfied 
	\begin{enumerate}
		\item $f(x)$ is lower bounded;
		\item There exists $\tau>0$, such that $\|\nabla \hat{f}(x)-\nabla f(x)\| \leq \tau$ almost surely;
	\end{enumerate} then the update of stochastic gradient descent converges almost surely with finite iteration complexity.
	\label{lemma:clippedsgd}
\end{lemma}
Then we continue to finish the multi-timescale convergence:
\begin{remark}
In each timescale, we prove two facts, (1) the stochastic error is bounded, so the clipped gradient descent will converge. (2) the converged point is a stationary point. \citep{bhatnagar2009natural,bhatnagar2012online,chow2017risk}.
\end{remark}
\noindent\textbf{Timescale 1} (Convergence of $\theta$ update). 

\noindent\textbf{Bounded error.} As the random variable $G_\theta$ depends on the state-action pair sampled from replay buffer, so in the following derivation, it is denoted as $G_\theta(s, a)$. For the SGD using sampled $(s_k, a_k)$ at $k^{\text{th}}$ step, the stochastic gradient is denoted by $G_{\theta_k}(s_k, a_k)$. The error term with respect to $\theta$ is computed by
\begin{equation}
	\begin{aligned}
		\delta \theta_{k} = G_{\theta_k}(s_k, a_k) - \mathbb{E}_{s\sim f^{\pi_{\theta}}}\big\{\mathbb{E}_{a\sim\pi_\theta} G_{\theta_k}(s, a)\big\}
	\end{aligned}
\end{equation}
Therefore, the error term is bounded by
\begin{equation}
	\begin{aligned}
		&\left\|\delta \theta_{k}\right\|^{2}\\ \leq &
		2 \left\|f^{\pi_\theta}(s)\pi(a|s)\right\|_\infty^2\left(\left\|G_{\theta_k}(s, a)\right\|_\infty^2+|G_{\theta_k}(s_k,a_k)|^2\right)\\
		\leq & 6 \left\|f^{\pi_\theta}(s)\pi(a|s)\right\|_\infty^2\left\|G_{\theta_k}(s, a)\right\|_\infty^2
	\end{aligned}
	\label{eq: sequaretheta}
\end{equation}
where $f^{\pi_\theta}$ is the state distribution density under $\pi_\theta$. As we assume the state and action are sampled from a closed set and the continuity of neural network, the upper bound is valid. According to Lemma \ref{lemma:clippedsgd} and invoking Theorem 2 in Chapter 2 of Borkar's book \cite{borkar2009stochastic}, the optimization converges to a fixed point $\theta^*$ almost surely (for given $\xi, \zeta$).

\noindent\textbf{Stationary point $\theta^*$.} Then we show that the fixed point $\theta^*$ is a stationary point using Lyapunov analysis. The analysis of the fastest timescale, $\theta$ update, is rather easy, but it is helpful for the similar analyses in the next two timescales. According to \citet{borkar2009stochastic}, we can regard the stochastic optimization of $\theta$ as a stochastic approximation of a dynamic system for given $\xi, \zeta$:
\begin{equation}
	\dot \theta = - \nabla_\theta \Ll'(\theta, \xi, \zeta)
	\label{eq:odetheta}
\end{equation}
\begin{proposition}
	consider a Lyapunov function for dynamic system \eqref{eq:odetheta}: 
	\begin{equation}
		L_{\xi, \zeta}(\theta)=\Ll'\left(\theta, \xi, \zeta\right)-\Ll'\left( \theta^{*}(\xi, \zeta), \xi, \zeta\right)
	\end{equation}
	where $\theta^*(\xi, \zeta)$ is a local minimum for given $\xi, \zeta$. In order to show that $\theta^*$ is a stationary point, we need
	\begin{equation}
		dL_{\xi, \zeta}(\theta)/dt\leq0
	\end{equation}
\end{proposition}
\begin{proof}
	We have 
	\begin{equation}
		\frac{d L_{\xi, \zeta}(\theta)}{d t}=-\left\|\nabla_{\theta} \Ll'(\theta, \xi, \zeta)\right\|^{2} \leq 0 \label{eq:descent}
	\end{equation}
	The equality holds only when $\Ll'(\theta, \xi, \zeta) =0$.\footnote{Similar convergence proof in \citet{chow2017risk} assumes that $\theta\in\Theta$ is a compact set, so they spend lots of effort to analyze the case when the $\theta$ reaches the boundary of $\Theta$. However, the proposed clipped SGD has released the requirements of compact domain of $\theta$, so the Lyapunov analysis becomes easier for the first timescale.} 	
\end{proof}
Combined with the conclusion of convergence, $\{\theta_k\}$ converges almost surely to a local minimum point $\theta^*$ for given $\xi$.

\vspace{5pt}
\noindent\textbf{Timescale 2} (Convergence of $\lambda$ update).

\vspace{2pt}
\noindent\textbf{Bounded Error.}

The error term of the $\xi$ update,
\begin{equation}
	G_{\xi_k}(s_k, a_k) - \mathbb{E}_{s}\big\{\mathbb{E}_{a} G_{\xi_k}(s, a)\big\}
\end{equation}
includes two parts: 
\begin{enumerate}
	\item $\delta\theta_\epsilon$  cased by inaccurate update of $\theta$ ($\theta$ should converge to $\theta^*(\xi, \zeta)$ in Timescale 1, but to $\theta_k$ near $\theta^*(\xi, \zeta)$):
	\begin{equation}
		\begin{aligned}
			\delta \theta_\epsilon = & {\hat \nabla}_{\xi} \Ll'(\theta_k, \xi, \zeta) - {\hat \nabla}_{\xi} \Ll'(\theta^*(\xi, \zeta), \xi, \zeta)\\
			= & \left(Q^{\pi_{\theta_k}}(s_k,a_k) - Q^{\pi_{\theta^*}}(s_k,a_k)\right)\nabla_\xi\lambda_\xi(s_k)\\
			= & \left(\nabla_aQ(s_k, a_k)\nabla_\theta\pi(s_k)\epsilon_{\theta_k} +o(\left\|\epsilon_{\theta_k}\right\|) \right)\nabla_\xi\lambda_\xi(s_k)
		\end{aligned}
		\label{eq:thetaerrorts2}
	\end{equation}
	Therefore, $\|\delta\theta_\epsilon\|\to 0$ as $\|\epsilon_\theta\|\to0$. The error is bounded since $\|\epsilon_\theta\|$ is a small error, where there must exists a positive scalar $\epsilon_0$ s.t. $\|\epsilon_\theta\|\leq \epsilon_0$. 
	\item $\delta \xi_{k}$ caused by estimation error of $\xi$:
	\begin{equation}
		\delta \xi_{k} = G_{\xi_k}(s_k, a_k) - \mathbb{E}_{s\sim f^{\pi_\theta}}\big\{\mathbb{E}_{a\sim\pi_\theta} G_{\xi_k}(s, a)\big\}
	\end{equation}
	\begin{equation}
	    \left\|\delta \xi_{k}\right\|^{2}\leq 
				4 \left\|f^{\pi_\theta}(s)\pi(a|s)\right\|_\infty^2\left(\max \left\|Q_\phi\left(s, a\right)\right\|^2+d^2\right)\left\|\nabla_\xi\lambda_\xi(s_t)\right\|_\infty^2
	\end{equation}
	Similar to the analysis of Timescale 1 with compact domain of $s_k$, we can get the valid upper bound.
\end{enumerate}

We again use Lemma \ref{lemma:clippedsgd} and Theorem 2 in Chapter 6 in \citet{borkar2009stochastic} to show that the sequence $\{\xi_k\}$ converges to the solution of the following ODE:
\begin{equation}
	\dot{\xi} = - \nabla_\xi \Ll'(\theta^*(\xi, \zeta), \xi, \zeta)
\end{equation}

\noindent\textbf{Stationary point.} Then we show that the fixed point $\xi^*$ is a stationary point using Lyapunov analysis. Note that we have to take $\epsilon_\theta$ into considerations.
\begin{proposition}
	For the dynamic system with the error term
	\begin{equation}
		\dot{\xi} = - \nabla_\xi\Ll'(\theta^*(\xi, \zeta)+\epsilon_{\theta}, \xi, \zeta)
	\end{equation}
	Define a Lyapunov function to be
	\begin{equation}
		L_\zeta(\xi)=\Ll'(\theta^*(\xi, \zeta), \xi, \zeta)-\Ll'\left( \theta^{*}(\xi, \zeta), \xi^{*}(\zeta), \zeta\right)
	\end{equation}
	where $\xi^*$ is a local maximum point. Then $\frac{d	L_\zeta(\xi)}{dt}\leq 0$.
	\label{prop:lyapunov2}
\end{proposition}
\begin{proof}
	The proof is similar to Proposition 2; only the error of $\theta$ should be considered. We prove that the error of $\theta$ does not affect the decreasing property here:
	\begin{small}
		\begin{equation}
			\begin{aligned}
				&\frac{dL_{\zeta}(\xi)}{dt} =  - \left(\nabla_\xi \Ll'(\theta^*(\xi, \zeta)+\epsilon_{\theta}, \xi, \zeta)\right)^T \nabla_\xi \Ll'(\theta^*(\xi, \zeta), \xi, \zeta) \\
				= & - \left\|\nabla_\xi \Ll'(\theta^*(\xi, \zeta), \xi, \zeta)\right\|^2 - \delta\theta_\epsilon^T \nabla_\xi \Ll'(\theta^*(\xi, \zeta), \xi, \zeta)\\
				\leq & - \left\|\nabla_\xi\Ll'(\theta^*(\xi, \zeta), \xi, \zeta)\right\|^2 + K_1 \left\|\epsilon_\theta\right\|\left\|\nabla_\xi \Ll'(\theta^*(\xi, \zeta), \xi, \zeta)\right\|\\
			\end{aligned}
		\end{equation}
	\end{small}
	where $K_1 = \|\nabla_aQ(s_k, a_k)\nabla_\theta\pi(s_k)\epsilon_{\theta_k}\|_\infty\|\nabla_\xi\lambda_\xi(s_k)\|_\infty\leq\infty$ according to \eqref{eq:thetaerrorts2} since the compact domain of $s_k,a_k$ according to Assumption \ref{assump:compact}.
	As $\theta$ converges much faster than $\xi$ according to the multiple timescale convergence in \citet{borkar2009stochastic}, we get $dL_{\zeta}(\xi)/dt\leq0$. Therefore, there exists trajectory $\xi(t)$ converges to $\xi^*$ if initial state $\xi_0$ starts from a ball $\mathcal{D}_{\xi^*}$ around $\xi^*$ according to the asymptotically stable systems.
\end{proof}
\noindent\textbf{Local saddle point of $(\theta^*, \xi^*)$.} One side of the saddle pioint, $\Ll'(\theta^*(\xi, \zeta), \xi^*(\zeta), \zeta)\leq \Ll'\left( \theta, \xi^*(\zeta), \zeta\right)$ are already provided in previous, so we need to prove here $\Ll'\left( \theta^*(\xi, \zeta), \xi^*(\zeta), \zeta\right)\geq \Ll'\left( \theta^*(\xi, \zeta), \xi, \zeta\right)$. 
To complete the proof we need that 
\begin{equation}
	Q_\phi(s,a)\leq d \text{ and } \lambda^*(s)(Q_\phi(s,a)-d)=0
\end{equation}
for all $s$ in $\Ss_f$, and $a$ sampled from $\pi_{\theta^*(\xi, \zeta)}$. Recall that $\lambda^*$ is a local maximum point, we have
\begin{equation}
	\nabla_\xi \Ll'\left( \theta^*(\xi, \zeta), \xi^*(\zeta), \zeta\right) = 0
\end{equation}
Assume there exists $s_t$ and action $a_t$ sampled from $\pi^*(s_t)$ so that $Q_\phi(s_t,a_t) > 0$. Then for $\lambda^*$ we have 
\begin{equation}
	\nabla_\xi\Ll'\left( \theta^*(\xi, \zeta), \xi^*(\zeta), \zeta\right)= d_\gamma^{\pi_\theta^*}(s_t)Q_\phi(s_t,a_t)\nabla_\xi\lambda_\xi(s_t) \neq 0
\end{equation}
The second part only requires that $\lambda^*(s_t)=0$ when $Q_\phi(s_t,a_t) < 0$. Similarly, we assume that there exists $s_t$ and $\xi^*$ where $\lambda_{\xi^*}(s_t) > 0$ and $Q_\phi(s_t,a_t) < d$. 
there must exists a $\xi_0$ subject to
\begin{equation}
	\xi_0 = \xi^* + \eta_0 f^{\pi_{\theta^*}}(s_t)\left(Q_\phi(s_t,a_t) - d\right)\nabla_\xi\lambda_\xi(s_t)
\end{equation}
for any $\eta\in(0,\eta_0]$ where $\eta_0\geq0$. It contradicts the statement the local maximum $\xi^*$.
Then we get
\begin{equation}
	\begin{aligned}
		& \Ll'\left( \theta^*(\xi, \zeta), \xi^*(\zeta), \zeta\right)\\
		=&  J_r(\theta^*) + \mathbb E_{s}\left\{\lambda_{\xi^*}(s)\mathbb{E}_{a\sim\pi}\{Q_\phi(s,a)\}\right\} =J_r(\theta^*)  \\
		\geq&   J_r(\theta^*) +  \mathbb E_{s}\left\{\lambda_{\xi}(s)\mathbb{E}_{a\sim\pi}\{Q_\phi(s,a)\}\right\} \\
		=&  \Ll'\left( \theta^*(\xi, \zeta), \xi, \zeta\right)
	\end{aligned}
\end{equation}
So $(\theta^*, \xi^*)$ is a locally saddle point for given safety index parameters $\zeta$.

\vspace{5pt}
\noindent \textbf{Timescale 3} (Convergence of $\zeta$ update).
\vspace{2pt}

\noindent\textbf{Bounded error.} Similar to Timescale 2, the error of $\zeta$ update includes two parts
\begin{enumerate}
	\item $\delta\theta_{\epsilon}+\delta\xi_{\epsilon}$ caused by inaccurate update of $\theta, \xi$:
	\begin{equation}
		\begin{aligned}
			& \delta\theta_{\epsilon}+\delta\xi_{\epsilon}\\
			= & \lambda_{\xi_k}(s_k)\nabla_\xi \Delta\phi^{\pi_{\theta_k}}(s_k) - \lambda_{\xi^*}(s_k)\nabla_\zeta \Delta\phi^{\pi_{\theta^*}}(s_k)\\
			= & \lambda_{\xi_k}(s_k)\nabla_\zeta \Delta\phi^{\pi_{\theta_k}}(s_k) -\lambda_{\xi_k}(s_k)\nabla_\zeta \Delta\phi^{\pi_{\theta^*}}(s_k) \\
			& + \lambda_{\xi_k}(s_k)\nabla_\zeta \Delta\phi^{\pi_{\theta^*}}(s_k) -  \lambda_{\xi^*}(s_k)\nabla_\zeta \Delta\phi^{\pi_{\theta^*}}(s_k)\\
			= & \left(\lambda_{\xi_k}(s_k) - \lambda_{\xi^*}(s_k)\right)\nabla_\zeta \Delta\phi^{\pi_{\theta^*}}(s_k)
		\end{aligned}
		\label{eq:zetaerror}
	\end{equation}
	The first part after the second equal sign is neglected since $\theta$ permutation does not affect on the gradient of $\zeta$. Similar to derivation in \eqref{eq:thetaerrorts2}, we get $|\delta\theta_{\epsilon}+\delta\xi_{\epsilon}|\rightarrow 0$ as $(\theta, \xi)\rightarrow(\theta^*, \xi^*)$.
	\item $\delta \zeta_{k}$ caused by estimation error of $\zeta$:
	\begin{equation}
		\delta\zeta_{k} = \lambda_{\xi^*}(s_k)\nabla_\zeta \Delta\phi^{\pi_{\theta^*}}(s_k) - \mathbb{E}\left\{\lambda_{\xi^*}(s)\nabla_\zeta \Delta\phi^{\pi_{\theta^*}}(s)\right\}
	\end{equation}
	The bounded error can be obtained by
	\begin{equation}
		\begin{aligned}
			&\left\|\delta \zeta_{k}\right\|^{2}\\ 
			\leq &
			4 \left\|f^{\pi_\theta}(s)\pi(a|s)\right\|_\infty^2\max_{s\in\complement_{\mathcal{D}}\Ss_f}|\lambda_{\xi^*}(s)|^2\|\nabla_\zeta \Delta\phi^{\pi_{\theta^*}}(s)\|^2\\
		\end{aligned}
	\end{equation}
\end{enumerate}

Therefore, the $\zeta$-update is a stochastic approximation of the continuous system $\zeta(t)$, described by the ODE
For the dynamic system
\begin{equation}
	\dot{\zeta} = - \nabla_\zeta \Ll'(\theta, \xi, \zeta)|_{\theta = \theta^*(\xi,\zeta)+\epsilon_{\theta}, \xi = \xi^*(\zeta)+\epsilon_{\xi}}
\end{equation}
\noindent\textbf{Stationary point.} Define a Lyapunov function
\begin{equation}
	L(\zeta)= \Ll'\left(\theta^*(\xi,\zeta), \xi^*(\zeta), \zeta\right)- \Ll'\left( \theta^{*}, \xi^{*}, \zeta^*\right)
\end{equation}
where $\zeta^*\in Z$ is a local minimum point.
Then
\begin{equation}
	\begin{aligned}
		& \frac{dL(\zeta)}{dt} \\
		= & \small{\left(- \nabla_\zeta \Ll'(\theta^*(\xi,\zeta)+\epsilon_{\theta}, \xi^*(\zeta)+\epsilon_{\xi}, \zeta) \right)^T \nabla_\zeta \Ll'(\theta^*(\xi,\zeta), \xi^*(\zeta), \zeta)} \\
		\leq & -  \left\|\nabla_\zeta \Ll'(\theta^*(\xi,\zeta), \xi^*(\zeta), \zeta)\right\|^2 \\
		& + K_2 \left\|\epsilon_\xi\right\| \left\|\nabla_\zeta \Ll'(\theta^*(\xi,\zeta), \xi^*(\zeta), \zeta)\right\|
	\end{aligned}
\end{equation}
where $K_2=\|\nabla_\xi \lambda_\xi(s)\|_\infty\|\nabla_\zeta \Delta\phi^{\pi_{\theta^*}}(s)\|_\infty$ according to \eqref{eq:zetaerror}. The derivation is very similar to the conclusion in Proposition \ref{prop:lyapunov2}; the upper bound is valid since the compact domain of state and action.
As $\xi$ converges faster than $\zeta$, $dL_{\zeta}(\xi)/dt\leq0$, so there exists trajectory $\xi(t)$ converges to $\zeta^*$ if initial state $\zeta_0$ starts from a ball $\mathcal{D}_{\zeta^*}$ around $\zeta^*$ according to the asymptotically stable systems. 

Finally, we can conclude that the sequence $(\theta_k, \xi_k, \zeta_k)$, will converge to a locally optimal policy and multiplier tuple $(\theta^*, \xi^*)$ for a locally optimal safety index parameters, $\zeta^*$.

\section{Implementation Details}
\subsection{Codebase and Platforms}
Implementation of FAC-SIS, FAC with $\phi_h$ and $\phi_F$ are based on the Parallel Asynchronous Buffer-Actor-Learner (PABAL) architecture proposed by \citet{duan2021distributional}.\footnote{\url{https://github.com/mahaitongdae/Safety_Index_Synthesis}} All experiments are implemented on Intel Xeon Gold 6248 processors with 12 parallel actors, including 4 workers to sample, 4 buffers to store data and 4 learners to compute gradients. Implementation of other baseline algorithms are based on the code released by \citet{ray2019benchmarking}\footnote{\url{https://github.com/openai/safety-starter-agents}} and also a modified version of PPO\footnote{\url{https://github.com/ikostrikov/pytorch-a2c-ppo-acktr-gail}}.
\label{sec:implementationdetail}

\subsection{Baseline Algorithms}
The only difference between baseline algorithms, FAC with $\phi_h,\ \phi_F$ is that no $\zeta$ update step in Algorithm \ref{alg:facspi}.
\label{sec:citerules}
\section{Hyperparameters}
The neural network design and detailed hyperparameters are listed in Table \ref{table:hyper}.
\label{sec:para}

\section{Additional Experimental Results}
\label{sec:addexp}
\subsection{Experiments in Other Safety Gym environments}
\label{sec:addexpgym}
We select six different Safety Gym environments, and the results of four of them are listed in the experiment section. The results with the rest of the environments are demonstrated here:
\begin{figure}[h]
	\centering
	\includegraphics[width=0.32\linewidth]{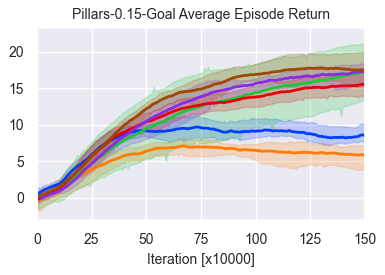}
	\includegraphics[width=0.32\linewidth]{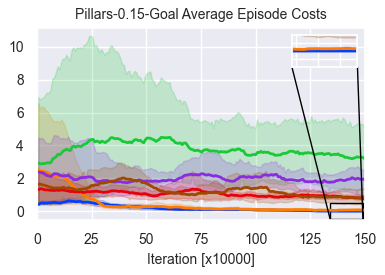}
	\includegraphics[width=0.32\linewidth]{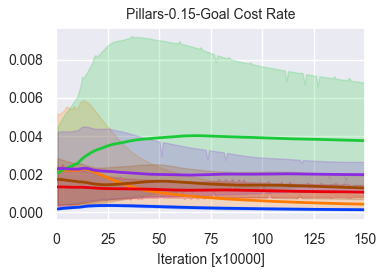}
	\includegraphics[width=0.32\linewidth]{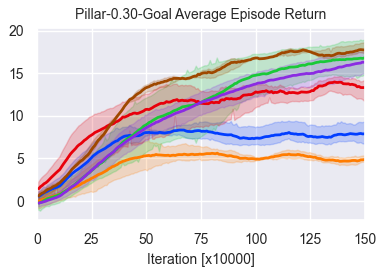}
	\includegraphics[width=0.32\linewidth]{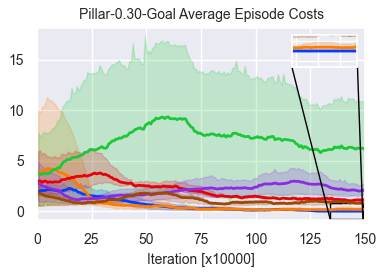}
	\includegraphics[width=0.32\linewidth]{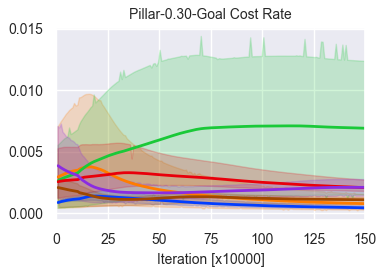}\\
	\subfigure{\includegraphics[width=0.7\linewidth]{fig/aaai_legends.png}}
	\caption{Average performance of FAC-SIS and baseline methods on 2 additional Safety Gym environments over five  seeds.}
	\label{fig:majorappendix}
\end{figure}

The results in the other two environments are consistent with the experiment section. The proposed FAC-SIS learns a safe policy with zero constraint violation, and other baseline algorithms all fail to neglect the cost even in the converged policies. Furthermore, the reward performance is better than the handcrafted safety index, or FAC with $\phi_h$.
\subsection{Custom Environment Details}
\label{sec:statedist}
To effectively scale the state distribution, we manually set an environment similar to the Safety Gym environments with \texttt{Point} robot, \texttt{Goal} task and \texttt{Hazard} obstacles shown in Figure \ref{fig:statedist}. The agent represented by the arrow should head to the red dot on the top while avoiding the hazard randomly located near the origin. The custom environment includes a static goal point at $(0, 5)$, a hazard with a radius of 0.5, and a point agent with two inputs, rotation and acceleration, the arrow represents the positive direction of the acceleration. The random initial positions (including positions, heading angles of the agent and the positions of hazard) design of three different distributions are listed in Table \ref{table:dist}. The reward design includes two parts that are the tracking error of the heading angle towards goal position, and speed relevant to the distance to the goal position (It requires that the agent always takes 5 seconds to reach the goal, so the closer the agent get to the goal, the slower its target speed is.).
\begin{figure}[htb]
	\centering
	\includegraphics[width=0.18\linewidth]{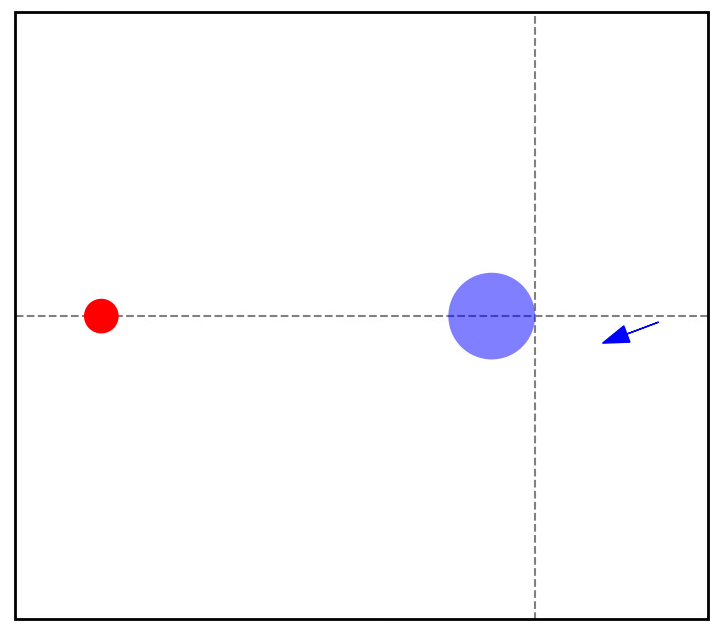}
	\includegraphics[width=0.24\linewidth]{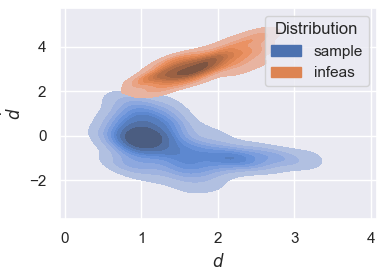}
	\includegraphics[width=0.24\linewidth]{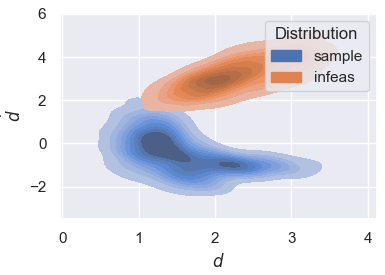}
	\includegraphics[width=0.24\linewidth]{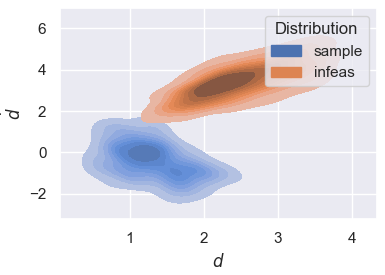}
	\caption{\small The custom environment and distributions of sampling and infeasible states under three different initialization setups. The overlap of the two distributions is very small, indicating that safe control exists for almost all sampled states.}
	\label{fig:statedist}
\end{figure}
The method to locate infeasible states is explained as follows. First, we discretize the state and action space with small intervals, and we exhaust all the actions to see if the energy can dissipate for each state. If not, we remark the specific state as infeasible.
\begin{table}[htbp]
\centering
	\begin{tabular}{@{}cccccc@{}}
		\toprule
		\multirow{2}{*}{\begin{tabular}[c]{@{}c@{}}Distribution\\ Index\end{tabular}} & \multicolumn{3}{c}{Agent Initial Position} & \multicolumn{2}{c}{Hazard Initial Position}       \\ \cmidrule(l){2-6} 
		& x            & y             & Angle   & \quad x\quad            & y          \\ \midrule
		1                                                                             & $0$          & $[-1.5,-1.0]$ & $[-\pi/4, \pi/4]$ & $0$ & $[0.5, 1]$\\
		2                                                                             & $0$          & $[-1.5,-0.5]$ & $[-\pi/4, \pi/4]$ &$0$&$[0.5, 1,5]$\\
		3                                                                             & $[-0.5,0.5]$ & $[-1.5,-0.5]$ & $[-\pi/4, \pi/4]$ &$0$&$[0.5, 1]$\\ \bottomrule
	\end{tabular}
	\caption{Initial distribution in custom environments. The initial position is sampled with a uniform distribution between the interval in the table.}
	\label{table:dist}
\end{table}

\subsection{Additional Results for Safety Index Synthesis}
\begin{table}[h]
	\centering
	\begin{tabular}{ccccc}
		\hline
		Parameters  & Notion   & $k$    & $\sigma$ & $n$   \\ \hline
		Synthesized & $\phi_{SIS}$ & 0.7821 & 0.0958   & 1.149 \\
		Handcrafted & $\phi_h$ & 1      & 0.3      & 2     \\
		Feasible    & $\phi_F$ & 1      & 0.04     & 2     \\
		Zero        & $\phi_0$ & 0      & 0        & 1     \\ \hline
	\end{tabular}
	\caption{Parameterization of Different Safety Index.}
	\label{table:paras}
\end{table}
We list the different safety index parameters in TABLE \ref{table:paras}. 
To be more specific, FAC-SIS has a single-direction update of the parameters by reducing the $k, \sigma$ and changing $n$ as shown in Figure \ref{fig:synthesis}.
\begin{figure}[h]
	\includegraphics[width=0.325\linewidth]{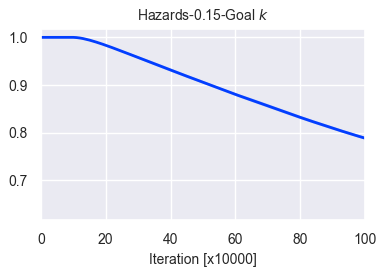}
	\includegraphics[width=0.325\linewidth]{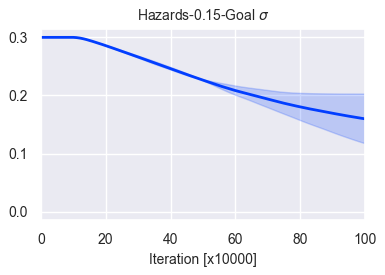}
	\includegraphics[width=0.325\linewidth]{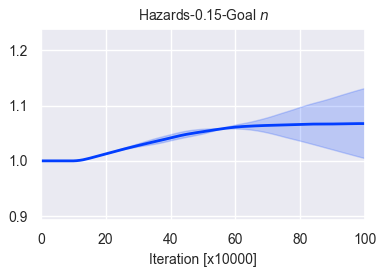}
	\caption{Learning curves of safety index parameters.}
	\label{fig:synthesis}
\end{figure}
 The trend can be explained by these two cases:
\begin{itemize}
	\item $\phi(s)\leq0$. Then the constraints, or the violation part are 
	\begin{equation}
		\phi(s')=\sigma + d_{min}^n-d^n - k \dot d\leq 0
		\label{eq:case1}
	\end{equation}
	Therefore, if we want to further optimize the violation part $J_\phi$, or the sum of LHS of the inequality in \eqref{eq:case1}, then we have:
	\begin{equation}
		\begin{array}{l}
			\partial J_\phi/\partial\sigma = 1>0\\
			\partial J_\phi/\partial n = n(d_{min}^{n-1}-d^{n-1})<0\\ 
			\partial J_\phi/\partial k = -\dot d
		\end{array}
	\end{equation}
	So, we should reduce $k$, increase $n$. The trend of $k$ depends on the $\dot d$.
	\item $\phi(s)>0$. If we further let $\eta_D=0$, then we have $\phi(s')\leq\phi(s)$. The violation part of this inequality constraints are:
	\begin{equation}
		d^n-d'^n+k\dot d-k\dot d'
		\label{eq:case2}
	\end{equation}
	we have 
	\begin{equation}
		\begin{array}{l}
			\partial J_\phi/\partial\sigma = 0\\
			\partial J_\phi/\partial n = n(d^{n-1}-d'^{n-1})\\ 
			\partial J_\phi/\partial k = \dot d-\dot d'
		\end{array}
	\end{equation}
	As the inequality constraints are violated, so at least there exists one positive term of $(d-d')$ and $(\dot d - \dot d')$. In other words, if we want to reduce the violation part \eqref{eq:case2}, whether that one of the $k$ or $n$ reduces, or they all reduce.
\end{itemize}
Therefore, the synthesis trends in Figure \ref{fig:synthesis} is reasonable.
\begin{table}[h]
    \centering
	\begin{tabular}{cccccc}
		\hline
		Metric                  & $\phi_0$ & $\phi_i$ & $\phi_{SIS}$ & $\phi_{F}$& $ \phi_{h}$  \\ \hline
		Success rate            & 0\%      & 85\%     & 99\%     & 100\% & 70\% \\
		$\phi_0$ violation rate & 100\%    & 0\%     & 0\%      & 0\%& 0\%\\
		Infeasible rate         & 100\%    &  15\%     & 1\%      & 0\%& 30\%\\
		Average Tracking Error  & -        & 2.974        & 3.378   & 3.456 & 3.528\\ \hline
	\end{tabular}
	\caption{Performance Comparison of Different Safety Indexes.}
	\label{table:per}
\end{table}

We add quantified metrics on the custom environment to compare the feasibility, safety, and optimality of different safety indexes shown in TABLE \ref{table:per}. We randomly simulate 100 trajectories to see if the safety index leads to the infeasibility of unsafe actions. If the agent does not violate $\phi_0$ (or stepping into hazards) or has infeasible states, then the trajectory is successful. $\phi_i$ is the initial safety index before synthesizing, and $\phi_{SIS}$ is the synthesized safety index. They both ensure safety but the $\phi_{SIS}$ increases the feasibility by the safety index synthesis. The $1\%$ difference might be caused by the mismatch of the state distributions between the custom environment and RL environments. $\phi_F$ indeed guarantees the best feasibility from the synthesis rules in \cite{anonymous2021modelfree}. Besides, the synthesized safety index also has slightly better optimality, although we do not intend to improve it. This could be explained by the synthesized safety index being less conservative since it only ensures the sampled region by RL is feasible but not all the state space like the synthesis rules in \cite{anonymous2021modelfree}.

\begin{table*}[htp]
	\vskip 0.15in
	\begin{center}
		\begin{tabular}{lc}
			\toprule
			Algorithm & Value \\
			\hline
			\emph{FAC-SIS, FAC w/ $\phi_h$, FAC w/ $\phi_F$} & \\
			\quad Optimizer &  Adam ($\beta_{1}=0.9, \beta_{2}=0.999$)\\
			\quad Approximation function  &Multi-layer Perceptron \\
			\quad Number of hidden layers & 2\\
			\quad Number of hidden units per layer & 256\\
			\quad Nonlinearity of hidden layer& ELU\\
			\quad Nonlinearity of output layer& linear\\
			\quad Actor learning rate & Linear annealing $3{\rm{e-}}5\rightarrow1{\rm{e-}}6 $\\
			\quad Critic learning rate & Linear annealing $8{\rm{e-}}5\rightarrow1{\rm{e-}}6 $\\
			\quad  Learning rate of multiplier net & Linear annealing $5{\rm{e-}}6\rightarrow5{\rm{e-}}6 $ \\
			\quad  Learning rate of $\alpha$ & Linear annealing $8{\rm{e-}}5\rightarrow8{\rm{e-}}6 $ \\
			\quad  Learning rate of safety index parameters (FAC-SIS only) & Linear annealing $8{\rm{e-}}6\rightarrow1{\rm{e-}}6 $ \\
			\quad Reward discount factor ($\gamma$) & 0.99\\
			\quad Policy update interval ($m_\pi$) & 3\\
			\quad  Multiplier ascent interval ($m_\lambda$)& 12\\
			\quad  SIS interval ($m_\phi$)& 24\\
			\quad Target smoothing coefficient ($\tau$) & 0.005\\
			\quad Max episode length ($N$) & \\
			\quad\quad Safety Gym task & 1000\\
			\quad\quad Custom task & 120 \\
			\quad  Expected entropy ($\overline{\mathcal{H}}$) &  $\overline{\mathcal{H}}=-$Action Dimensions \\
			\quad  Replay buffer size & $5\times10^5$\\
			\quad  Replay batch size & 256\\
			\quad Handcrafted safety index ($\phi_h$) hyperparameters $(\eta, n, k, \sigma)$ & $(0,\ 2,\ 1,\ 0.3)$\\
			\quad Feasible safety index ($\phi_F$) hyperparameters $(\eta, n, k, \sigma)$ & $(0,\ 2,\ 1,\ 0.04)$\\\midrule
			\emph{CPO, TRPO-Lagrangian} &\\ 
			\quad Max KL divergence&  $0.1$\\
			\quad Damping coefficient&  $0.1$\\
			\quad Backtrack coefficient&  $0.8$\\
			\quad Backtrack iterations&  $10$\\
			\quad Iteration for training values&  $80$\\
			\quad Init $\lambda$ &  $0.268 (softplus(0))$\\
			\quad GAE parameters  &  $0.95$\\
			\quad Batch size &  $2048$\\
			\quad Max conjugate gradient iterations & $ 10$ \\
			
			\hline
			\emph{PPO-Lagrangian} &\\ 
			\quad Clip ratio &  $0.2$\\
			\quad KL margin &  $1.2$\\
			\quad Mini Bactch Size & $64$\\
			\bottomrule
		\end{tabular}
	\end{center}
	\vskip -0.1in
	\caption{Detailed hyperparameters.}
	\label{table:hyper}
\end{table*}



\end{document}